\documentclass[11pt]{article}

\usepackage[margin=1in]{geometry}
\usepackage[T1]{fontenc}
\usepackage[utf8]{inputenc}
\usepackage{lmodern}
\usepackage{microtype}
\usepackage{amsmath, amssymb, amsthm, mathtools}
\usepackage{booktabs}
\usepackage{array}
\usepackage{enumitem}
\usepackage[numbers,sort&compress]{natbib}
\usepackage{hyperref}
\usepackage{xcolor}
\usepackage{bbm}
\usepackage[ruled,vlined,linesnumbered]{algorithm2e}
\usepackage{tikz}
\usetikzlibrary{arrows.meta, positioning, shapes.geometric, fit, calc, decorations.pathreplacing}

\hypersetup{
  colorlinks=true,
  linkcolor=blue!50!black,
  citecolor=blue!50!black,
  urlcolor=blue!50!black,
  pdftitle={Cascade-Aware Multi-Agent Routing via Adaptive Geometry Selection},
  pdfauthor={Davide Di Gioia}
}

\newtheorem{proposition}{Proposition}
\newtheorem{definition}{Definition}

\newcommand{\E}{\mathbb{E}}
\newcommand{\Prob}{\mathbb{P}}
\newcommand{\R}{\mathbb{R}}
\newcommand{\I}{\mathbbm{1}}
\newcommand{\relu}{\operatorname{ReLU}}
\newcommand{\sigm}{\sigma}
\newcommand{\norm}[1]{\lVert #1\rVert}

\title{Cascade-Aware Multi-Agent Routing:\\Adaptive Geometry Selection on Live Execution Graphs}
\author{Davide Di Gioia\\\small\texttt{ucesigi@ucl.ac.uk}}
\date{March 2026}

\begin{document}

\maketitle

% ============================================================
%  ABSTRACT  (~250 words, down from ~400)
% ============================================================
\begin{abstract}
Advanced AI reasoning systems increasingly route tasks through dynamic execution graphs of specialized agents.
We identify a structural blind spot shared by this architectural class: their schedulers optimize load and fitness but carry no model of how failure propagates differently depending on whether the graph is tree-like, cyclic, or mixed.
In tree-like regimes a single agent failure cascades exponentially; in dense cyclic regimes the same failure self-limits within one hop.
A geometry-blind scheduler cannot distinguish these cases.

We formalize this \emph{failure-propagation observability gap} as a runtime decision problem: which propagation model to trust at each scheduling step. We then prove a cascade-sensitivity condition: a failure cascade becomes supercritical when per-edge propagation probability exceeds the inverse of the graph's branching factor ($p > e^{-\gamma}$, where $\gamma$ is the BFS shell-growth exponent).
We close the gap with a lightweight add-on risk-scoring module, a \emph{spatio-temporal sidecar} that reads the live execution graph and recent failure history to predict which routing geometry best fits the current structural regime.
Concretely, the sidecar comprises (i)~a Euclidean propagation scorer that models failure spread in dense, cyclic subgraphs, (ii)~a negatively-curved (hyperbolic) scorer that captures exponential risk amplification in tree-like subgraphs, and (iii)~a compact learned gate ($9 \!\to\! 12 \!\to\! 1$ MLP, 133 parameters) that blends the two scores based on six topology statistics and three geometry-aware features.

On 250 benchmark scenarios spanning five topology regimes, the sidecar lifts the native scheduler's win rate from $50.4\%$ to $87.2\%$ ($+36.8$\,pp).
In tree-like regimes the gain reaches $+48$ to $+68$\,pp, matching the cascade-sensitivity prediction.
The learned gate achieves held-out AUC\,$=0.9247$, accuracy\,$86.4\%$, and ECE\,$0.0681$, confirming that geometry preference is recoverable from live structural signals.
Cross-architecture validation on Barab\'{a}si--Albert, Watts--Strogatz, and Erd\H{o}s--R\'{e}nyi graphs confirms that propagation modeling generalizes across graph families.
\end{abstract}

% ============================================================
%  1. INTRODUCTION
% ============================================================
\section{Introduction}

A prevailing approach to building advanced AI reasoning systems is orchestrating specialized agents, such as planners, generators, critics, and verifiers, into dynamic pipelines.
Frameworks such as LangGraph, AutoGen~\citep{wu2023autogen}, and CrewAI, modular mixture-of-experts routing layers, and self-improving multi-agent engines all share the same underlying architecture: agents are nodes, task delegations are directed edges, and executing a task means routing a signal through a continuously evolving graph.
The same structure appears in advanced neuro-symbolic reasoning systems such as AlphaGeometry~\citep{trinh2024} and OpenCog Hyperon~\citep{goertzel2023}.

This architectural consensus creates a shared, unexamined blind spot.
No matter how sophisticated the individual agents, the scheduler that routes tasks is \emph{geometry-blind}: it optimizes based on node load and agent fitness, but carries no model of how the graph's overarching shape dictates failure propagation.
In tree-like delegation regimes, a single agent failure cascades exponentially through the branching structure.
In dense, cyclic regimes, the same failure is safely absorbed within one or two hops.
Because current schedulers cannot perceive this difference in graph geometry, they cannot proactively avoid catastrophic cascades.

\textit{Intuition.}
Think of a navigation system that routes traffic based only on current road congestion.
If the city is laid out as a grid, a blockage on one street has a small footprint; surrounding blocks offer immediate detours.
But if the city is fed by a single arterial highway with few exits, the same blockage halts all downstream traffic.
A congestion-only router cannot distinguish the two layouts; it needs a model of the road topology, not just the traffic.
Geometry-blind multi-agent schedulers face the same problem: node load and fitness are the congestion signal, but the execution graph's shape, tree-like or cyclic, determines how far a single agent failure travels.

\begin{figure}[t]
\centering
\begin{tikzpicture}[
  >=Stealth,
  safe/.style={circle, draw=black!50, fill=green!20, thick, minimum size=7mm, inner sep=0pt},
  failed/.style={circle, draw=red!80, fill=red!30, very thick, minimum size=7mm, inner sep=0pt},
  affected/.style={circle, draw=orange!80, fill=orange!25, very thick, minimum size=7mm, inner sep=0pt},
  gedge/.style={draw=gray!50, thick},
  cascade/.style={->, draw=red!65, very thick},
]
% ── (a) TREE-LIKE: 7 nodes, failure cascades everywhere ──
\begin{scope}
  \node[failed]   (A) at (3.0, 0.0)  {}; % root - failure injected
  \node[affected] (B) at (1.5,-1.3)  {};
  \node[affected] (C) at (4.5,-1.3)  {};
  \node[affected] (D) at (0.5,-2.7)  {};
  \node[affected] (E) at (2.5,-2.7)  {};
  \node[affected] (F) at (3.5,-2.7)  {};
  \node[affected] (G) at (5.5,-2.7)  {};
  \draw[cascade] (A)--(B); \draw[cascade] (A)--(C);
  \draw[cascade] (B)--(D); \draw[cascade] (B)--(E);
  \draw[cascade] (C)--(F); \draw[cascade] (C)--(G);
  \node[red!70!black, font=\footnotesize\bfseries] at (3.0, 0.6) {failure};
  \node[font=\small, align=center] at (3.0,-3.5)
    {\textbf{(a) Tree-like execution graph}\\
     {\footnotesize All 6 downstream agents affected}};
\end{scope}
% ── (b) DENSE/CYCLIC: 7 nodes, failure self-limits ──
\begin{scope}[xshift=8.2cm]
  \node[failed]   (a) at (1.5, 0.0)  {};
  \node[affected] (b) at (3.2,-0.8)  {};
  \node[safe]     (c) at (3.5,-2.3)  {};
  \node[safe]     (d) at (2.0,-3.1)  {};
  \node[safe]     (e) at (0.0,-3.1)  {};
  \node[safe]     (f) at (-0.5,-1.6) {};
  \node[safe]     (g) at (0.5,-0.8)  {};
  % Dense cross-edges provide alternative paths
  \draw[gedge] (a)--(g); \draw[gedge] (a)--(f);
  \draw[gedge] (b)--(c); \draw[gedge] (b)--(g); \draw[gedge] (b)--(d);
  \draw[gedge] (c)--(d); \draw[gedge] (c)--(e);
  \draw[gedge] (d)--(e); \draw[gedge] (e)--(f); \draw[gedge] (f)--(g);
  \draw[gedge] (g)--(d);
  % Only one cascade hop
  \draw[cascade] (a)--(b);
  \node[red!70!black, font=\footnotesize\bfseries] at (1.5, 0.6) {failure};
  \node[font=\small, align=center] at (1.5,-3.9)
    {\textbf{(b) Dense/cyclic execution graph}\\
     {\footnotesize 5/6 agents unaffected: alternative paths absorb spread}};
\end{scope}
\end{tikzpicture}
\caption{The geometry-blindness problem. The same single-agent failure (red) injected into a tree-like execution graph cascades exponentially through all downstream delegation chains~(a), while in a dense, cyclic graph (b) it is absorbed within one hop because alternative paths reroute around the failure. A scheduler that sees only node load and fitness cannot distinguish these two cases; our sidecar detects the relevant structural difference, branching rate vs.\ cycle density, from live graph features and adjusts its risk model accordingly.}
\label{fig:toy}
\end{figure}
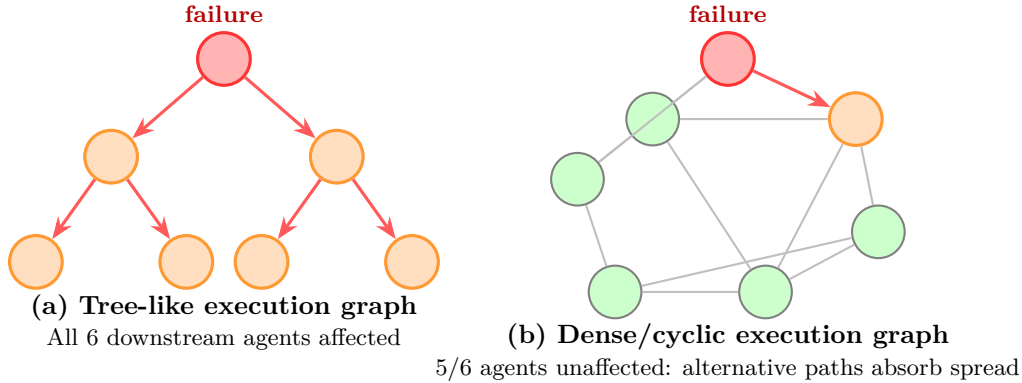

\paragraph{Research question.}
Given an AI system whose execution graph dynamically shifts between tree-like, cyclic, and mixed regimes, how can a scheduler recognize its current topology and adapt its routing strategy to minimize failure cascade risk?

To our knowledge, this question is not posed in prior work.
Prior work in geometric graph learning establishes that different topologies require different geometric representations (e.g., hyperbolic space for trees, Euclidean space for dense graphs), but these geometries are fixed prior to training; they do not adapt to live, intra-episode structural shifts.
Multi-agent routing papers optimize scheduling and coordination without modeling graph geometry or failure propagation.
Neither addresses the runtime decision of which propagation model to trust as topology shifts, a problem that arises in any geometry-blind scheduler operating on a time-varying execution graph (formalized as an \emph{online geometry-control problem} in Section~\ref{sec:formulation}).

\newpage
\paragraph{Contributions.}
\begin{enumerate}[leftmargin=1.5em]
\item \textbf{Blind spot identification.} We identify geometry-blindness as a failure-propagation observability gap common to execution-graph agent systems (Section~\ref{sec:gap}).
\item \textbf{System-level cost.} We quantify the gap via a system-level ablation: load-and-fitness-only routing achieves only $50.4\%$ win rate, no better than chance under topology stress (Section~\ref{sec:ablation}).
\item \textbf{Deployable fix (primary contribution).} A \emph{structural baseline} (centrality + load + 2-hop neighborhood, zero trained parameters) immediately recovers $87.2\%$ win rate and transfers to out-of-distribution graph families without retraining (Section~\ref{sec:results}).
\item \textbf{Adaptive refinement (optional).} A learned geometry selector ($9\!\to\!12\!\to\!1$ MLP, 133 parameters) ties the structural baseline on win rate while improving route-separation margin by $+44\%$, offering an additional layer when fine-tuning data is available (Section~\ref{sec:results}).
\item \textbf{Theoretical grounding.} We prove a cascade-sensitivity condition (Proposition~\ref{prop:cascade}) whose regime predictions match the empirical curves, explaining \emph{when} and \emph{why} each model adds value (Section~\ref{sec:cascade}).
\end{enumerate}

\paragraph{Research questions \& hypotheses.}
\begin{itemize}[leftmargin=1.5em, itemsep=2pt]
\item \textbf{RQ1 / H1} (Perception gap): Does explicit failure-propagation modeling improve route selection relative to load-and-fitness-only routing?
\item \textbf{RQ2 / H2} (Regime adaptation): Under topology regime shift, does a learned geometry selector outperform any fixed geometric prior?
\item \textbf{RQ3 / H3} (Theory alignment): Do the largest gains occur in high-expansion tree-like regimes, as predicted by cascade-expansion analysis?
\end{itemize}

\paragraph{Paper outline.}
Section~\ref{sec:gap} reviews four literatures and identifies the geometry-blindness gap that none addresses.
Section~\ref{sec:formulation} formalizes the gap as an online geometry-control problem and argues that both \emph{spatial} (graph-structural) and \emph{temporal} (failure-history) information are needed to close it.
Section~4 builds the spatio-temporal route-risk model that couples temporal failure intensity with graph propagation; neither dimension alone is sufficient, as the ablation in Section~\ref{sec:decomposition} will confirm.
Section~5 adds adaptive geometry selection; Section~\ref{sec:cascade} grounds the gate in a cascade-sensitivity analysis.
Sections~7--8 present the experimental setup and results, including a component decomposition showing that temporal modeling without spatial structure is actively harmful.
Section~9 discusses implications; Section~\ref{sec:limitations} states limitations.

% ============================================================
%  2. RELATED WORK AND GAP
% ============================================================
\section{Related Work and Gap}
\label{sec:gap}

The present work lies at the intersection of four literatures.
Each contributes an important ingredient; none directly studies the question posed here.

\paragraph{Geometric graph learning.}
Hyperbolic representation learning is now standard for graphs with latent hierarchy \citep{ganea2018,chami2019}.
Mixed-curvature models showed that no single constant-curvature manifold is universally adequate \citep{skopek2020,bachmann2020}.
These papers establish that geometry matters, but geometry is chosen once for a dataset, not online.
Recent dynamic-hyperbolic work \citep{sun2021,yang2021} improves time-evolving graph modeling but still optimizes a chosen non-Euclidean family rather than treating geometry selection as the prediction problem.

\paragraph{Spatio-temporal \& dynamic graph learning.}
Canonical models such as DCRNN \citep{li2018}, STGCN \citep{yu2018}, and Graph WaveNet \citep{wu2019} combine graph propagation with temporal modeling.
Dynamic-graph methods \citep{trivedi2019,rossi2020,pareja2020,xu2020} focus on continuous-time representation learning.
These methods assume a fixed graph-learning architecture; they do not ask whether the underlying geometric prior should itself switch as topology changes.

\paragraph{Temporal point processes.}
Hawkes processes \citep{hawkes1971} and neural variants \citep{du2016,mei2017,zuo2020} model self-exciting dynamics.
\citet{zuo2020} encode inter-arrival patterns via Transformer self-attention over event histories; we use a simpler gap-based burst statistic as one ingredient in route-risk scoring.
The gap addressed here is not temporal dependence per se but its coupling with geometry selection, which Transformer Hawkes and its variants do not address.

\paragraph{Multi-agent communication \& orchestration.}
Learned communication and coordination \citep{sukhbaatar2016,das2019,jiang2018} and classical blackboard architectures \citep{hayesroth1985,durfee1991} study how agents exchange information over changing interaction structures, but do not build a route-risk estimator whose inductive bias changes with topology.

Recent work on LLM routing with bandit feedback \citep{wei2025barp} addresses a complementary problem: selecting \emph{which model} to call for a given query under a cost-quality trade-off.
In that framing, routing goes \emph{to} a model: a single-hop selection with no execution graph, no delegation chain, and no failure-propagation model.
Our problem is orthogonal: routing goes \emph{through} a multi-agent structure, and the question is which path through the live execution graph minimizes cascade risk given the graph's current topology.

The most directly relevant recent work is \citet{kim2025scaling}, who conduct a controlled evaluation of five canonical multi-agent topologies across 180 configurations and derive quantitative scaling laws.
Their key finding, that independent-agent coordination amplifies errors $17.2\times$ while centralized coordination contains this to $4.4\times$, directly motivates our concern.
A predictive coordination model using topology-class and task-property features achieves cross-validated $R^2 = 0.524$, confirming that topology is a strong predictor of system-level outcomes.
However, their analysis operates at \emph{design time}: topology is a fixed architectural choice.
It does not model the intra-episode structural shifts that occur as a live execution graph moves between regimes during a single solve.
Our work addresses exactly this runtime perception gap.

\begin{definition}[Geometry-blind routing]
\label{def:geometry_blind}
A system routing among candidate routes $r \in \mathcal{R}_t$ at time~$t$ using~$\mathrm{score}(r,t) = f(\psi_t(r), \xi_t)$, where $\psi_t(r)$ are route-level aggregate observables (load, fitness, cost) and $\xi_t$ are task features, is \emph{geometry-blind} if the scoring rule does not depend on any propagation-sensitive structural representation $\phi(G_t[r])$ measuring expansion, cyclicity, or hyperbolicity of the route subgraph.

\smallskip\noindent
\textit{Equivalently:} routing is geometry-blind iff, conditional on $(\psi_t(r), \xi_t)$, the routing decision is invariant to any structural modification of $G_t[r]$ that preserves those aggregate observables.
\end{definition}

\paragraph{Gap statement.}
Prior work does not identify geometry-blindness as a failure-propagation observability gap in execution-graph agent systems, nor provide a runtime mechanism to close it.
The individual components (hyperbolic structure, dynamic graphs, temporally clustered failures) are each well established; the contribution is their integration into a runtime geometry-control layer inside a live multi-agent scheduler.

Closing this gap requires reasoning about \emph{where} in the graph structure a failure sits (spatial) and \emph{when} failures cluster in time (temporal), simultaneously.
Neither dimension alone suffices: temporal models that ignore graph expansion cannot distinguish supercritical from self-limiting cascades; spatial models that ignore failure history cannot detect burst regimes.
We formalize this joint requirement next.

% ============================================================
%  3. PROBLEM FORMULATION
% ============================================================
\section{Problem Formulation}
\label{sec:formulation}

\subsection{System Context: Genesis~3}

Genesis~3 is a multi-agent reasoning engine in which specialized agents (planners, generators, critics, verifiers) collaborate through a 9-gate self-modification pipeline.
Each node $v \in V_t$ is an active agent instance; each directed edge $(u,v) \in E_t$ is a task-delegation dependency.
A \emph{route} $r = (v_1, \ldots, v_k)$ is an ordered delegation chain selected at each solve episode.
The 9-gate pipeline creates deep sequential dependencies, meaning failure at one gate may propagate forward and block all downstream stages. This is the key motivation for failure-propagation modeling.

Without the spatio-temporal sidecar, Genesis~3 selects routes using a \texttt{LinUCBTuner}: a contextual bandit with 9 route arms and a 6-dimensional context vector, scoring routes by team fitness and mean node load.
Concretely, the native score is $\mathrm{score}(r) = f_{\mathrm{team}}(r) \times (1 - 0.5\,\bar{\ell}(r))$, where $f_{\mathrm{team}}$ is an exponentially smoothed success rate and $\bar{\ell}$ is normalized mean load over the route.
This is adequate for steady-state load balancing but carries no failure-propagation model and no geometric prior.

\subsection{Abstract Model}

To formalize this routing problem, we need to represent both the \emph{spatial structure} of the multi-agent system, namely which agents exist and how they are connected, and the \emph{temporal history} of its failures, namely when and where things have gone wrong.
We model the running system as a time-indexed sequence of directed graphs
\[
\mathcal{G} = \{G_t = (V_t, E_t, w_t, \ell_t) : t \in \R_{\ge 0}\},
\]
where $w_t$ denotes edge reliability and $\ell_t$ denotes node load.
A route $r$ at time $t$ activates a subgraph $G_t[r] = (V_t[r], E_t[r])$.
Let the failure history be
\[
\mathcal{F} = \{(\tau_i, v_i, s_i, c_i, r_i)\}_{i=1}^N,
\]
where $\tau_i$ is time, $v_i$ is the affected node, $s_i \in [0,1]$ is severity, $c_i$ is event category, and $r_i$ is an optional route label.

The task is to construct a route-risk score $R(r, t) \in \R$ that predicts vulnerability under future failure propagation.
The central difficulty is that the correct geometric inductive bias is time-dependent.
We introduce a soft selector $\pi_t(r) \in [0,1]$, interpreted as preference for hyperbolic geometry, to be learned from graph structure and recent event history.

\subsection{Why Both Spatial and Temporal?}
\label{sec:why_st}

A natural question is whether one dimension suffices.
Consider each in isolation:

\begin{itemize}[leftmargin=1.5em, itemsep=3pt]
\item \textbf{Temporal only.}
A model that tracks when failures occur (e.g., exponentially decayed intensity) but ignores graph structure cannot distinguish a burst of failures on a leaf node (harmless) from the same burst at a high-degree hub (catastrophic).
It assigns identical risk to routes that traverse structurally different subgraphs, because it has no representation of expansion rate or cycle rank.

\item \textbf{Spatial only.}
A model that computes centrality, degree, or 2-hop neighborhoods but ignores failure timing cannot detect that a structurally safe route has entered a burst regime where temporally clustered failures are about to saturate a node's recovery capacity.
It treats a quiescent high-centrality node identically to one under active burst.

\item \textbf{Spatio-temporal.}
Coupling $\lambda_t(v;r)$ (temporal history per node) with graph-structural propagation (how that node's risk spreads through the route subgraph) allows the model to answer the operational question: \emph{this route traverses a high-expansion subgraph, and the nodes on it are currently experiencing temporally clustered failures. How much worse is that than the same failures on a cyclic subgraph?}
\end{itemize}

The ablation in Section~\ref{sec:decomposition} confirms this decomposition empirically: temporal-only modeling (Euclidean row in Table~\ref{tab:overall}) achieves $48.0\%$ win rate, \emph{worse} than the geometry-blind native baseline ($50.4\%$), because temporal propagation without spatial awareness adds noise rather than signal.
Adding spatial structure is the decisive intervention.

% ============================================================
%  ARCHITECTURE FIGURE
% ============================================================
\begin{figure}[t]
\centering
\begin{tikzpicture}[
  >=Stealth,
  box/.style={draw, rounded corners=3pt, minimum height=0.8cm, minimum width=2.2cm, align=center, font=\small},
  data/.style={draw, rounded corners=1pt, minimum height=0.6cm, minimum width=1.8cm, align=center, font=\footnotesize, fill=blue!5},
  gate/.style={draw, rounded corners=3pt, minimum height=0.8cm, minimum width=2.2cm, align=center, font=\small, fill=orange!15, thick},
  score/.style={draw, rounded corners=3pt, minimum height=0.8cm, minimum width=2.0cm, align=center, font=\small, fill=green!10},
  arrow/.style={->, thick, >=Stealth},
]

% Input data
\node[data] (graph) at (0, 0) {$G_t$\\Execution graph};
\node[data] (fails) at (3.5, 0) {$\mathcal{F}$\\Failure history};
\node[data] (route) at (7, 0) {$r$\\Candidate route};

% Feature extraction
\node[box, fill=blue!8] (feat) at (3.5, -1.6) {Feature map $\phi(G_t[r]) \in \R^9$\\{\footnotesize $\phi_1$--$\phi_6$: topology \quad $\phi_7$--$\phi_9$: geometry}};

% Two scoring branches
\node[score] (euc) at (0.5, -3.5) {$R_{\mathrm{Euc}}(r,t)$\\Euclidean};
\node[score] (hyp) at (6.5, -3.5) {$R_{\mathrm{Hyp}}(r,t)$\\Hyperbolic};

% Geometry gate
\node[gate] (gate) at (3.5, -3.5) {Geometry Gate\\$\pi_t(r)$\\{\footnotesize MLP $9\!\to\!12\!\to\!1$}};

% Blender
\node[box, fill=yellow!12, thick] (blend) at (3.5, -5.2) {$R(r,t) = \pi \cdot R_{\mathrm{Hyp}} + (1\!-\!\pi)\cdot R_{\mathrm{Euc}}$};

% Output
\node[box, fill=red!8] (out) at (3.5, -6.5) {Route-risk signal\\to scheduler};

% Arrows
\draw[arrow] (graph) -- (feat);
\draw[arrow] (fails) -- (feat);
\draw[arrow] (route) -- (feat);

\draw[arrow] (feat) -- (gate);
\draw[arrow] (feat) -| (euc);
\draw[arrow] (feat) -| (hyp);

\draw[arrow] (graph.south) -- ++(0,-0.4) -| (euc.north);
\draw[arrow] (fails.south) -- ++(0,-0.3) -| (hyp.north);

\draw[arrow] (euc) |- (blend);
\draw[arrow] (hyp) |- (blend);
\draw[arrow] (gate) -- (blend);

\draw[arrow] (blend) -- (out);

\end{tikzpicture}
\caption{Spatio-temporal sidecar architecture. The execution graph $G_t$, failure history $\mathcal{F}$, and candidate route $r$ feed into a 9-dimensional structural feature map $\phi$, which drives both the Euclidean/hyperbolic scoring branches and a learned geometry gate. The gate output $\pi_t(r)$ blends the two scores into a single route-risk signal for the scheduler.}
\label{fig:architecture}
\end{figure}
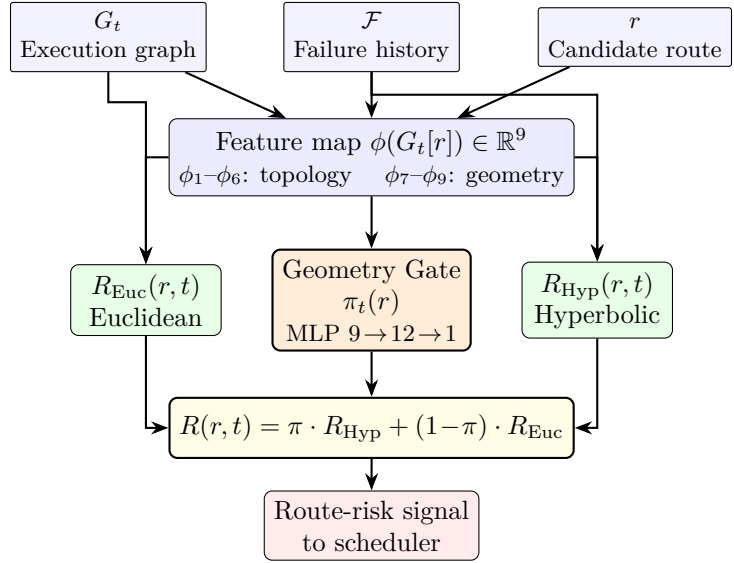

% ============================================================
%  4. SPATIO-TEMPORAL ROUTE-RISK MODEL
% ============================================================
\section{Spatio-Temporal Route-Risk Model}

\subsection{Temporal Failure Intensity}

For node $v$ at time $t$, define the temporally decayed failure intensity
\[
\lambda_t(v; r) = \sum_{i : v_i = v,\; \tau_i \le t}
\Bigl(
  \I\!\left[r_i \in \{\varnothing, r\}\right]\;
  s_i\;
  \exp\!\left(-\frac{t - \tau_i}{h}\right)\;
  \alpha(c_i)
\Bigr),
\]
where $h > 0$ is a half-life parameter and $\alpha(c_i)$ is a category-dependent multiplier.

The hyperbolic model augments this with a same-node burst term.
Let $m_v$ count relevant events at node $v$ up to time $t$, and $\tau_{(1)} < \cdots < \tau_{(m_v)}$ be the ordered failure times.
The gap-based burst statistic is
\[
b_t(v; r) =
\begin{cases}
\dfrac{1}{m_v-1}\displaystyle\sum_{q=1}^{m_v-1}
\exp\!\left(-\dfrac{\tau_{(q+1)} - \tau_{(q)}}{\delta}\right), & m_v \ge 2,\\[1.0ex]
0, & m_v < 2,
\end{cases}
\]
where $\delta > 0$ is an excitation-decay constant.
The damped intensity is then
\[
\tilde{\lambda}_t(v; r) = \lambda_t(v; r)
  + 0.14\,b_t(v; r)\,s_t(v; r)\,d_t(v; r)\,o_t(v; r),
\]
with saturation $s_t = \tanh(\bar{\lambda}_t)$, diversity $d_t = m_v^{-1/2}$, and overload factor $o_t = (1 + \max\{0, \lambda_t - \lambda_0\})^{-1}$,
where $\lambda_0 = 1$ is the baseline excitation threshold and
\[
\overline{\lambda}_t(v; r) := \frac{1}{m_v}\sum_{q=1}^{m_v}
s_{(q)}\,\exp\!\left(-\frac{t-\tau_{(q)}}{h}\right)\,\alpha\bigl(c_{(q)}\bigr)
\]
denotes the mean decayed severity of same-node events included in the base intensity.
This is not a full Hawkes intensity; it is a bounded burstiness bonus based on local inter-arrival times.

\subsection{Euclidean Route Propagation}

We use Euclidean propagation to model failure diffusion in dense, cyclic subgraphs, where risk spreads outward like heat through a grid: many short alternative paths limit individual failures to a small local footprint.

The Euclidean baseline propagates node-local risk over the route subgraph with diffusion, adaptive recovery, latency, and bottleneck penalties:
\[
x_{t+1}(u) = (1 - \rho_u)\,x_t(u) + \sum_{(v,u) \in E_t[r]} \eta_{vu}\,x_t(v),
\]
where $\rho_u$ is adaptive recovery and $\eta_{vu}$ absorbs diffusion strength, load, latency, reliability, and route-local amplification.
The Euclidean route score $R_{\mathrm{Euc}}(r,t)$ is a weighted combination of infected mass, frontier risk, tail risk, latency penalty, and bottleneck pressure.

\subsection{Hyperbolic Route Propagation}

We use a negatively-curved (hyperbolic) space to model failure propagation in tree-like subgraphs.
The key property of hyperbolic geometry is that its volume expands exponentially with radius, exactly mirroring how the number of reachable nodes grows with BFS depth in a branching structure.
This makes distances in hyperbolic space a natural measure of cascade reach: a failure near the root of a delegation tree is far from the leaves in Euclidean terms, but hyperbolic distances correctly reflect the exponentially large downstream exposure.

Concretely, the hyperbolic model embeds nodes into a Poincar\'{e} ball with curvature magnitude $\kappa > 0$:
\[
z_t(v) \in \mathbb{B}^{d}_{\kappa} := \bigl\{z \in \R^d : \norm{z} < 1/\sqrt{\kappa}\bigr\}.
\]
Routes are scored by geodesic quantities:
\[
R_{\mathrm{Hyp}}(r,t) =
  w_1 C_t(r)
  - w_2 T_t(r)
  - w_3 F_t(r)
  - w_4 B_t(r)
  + w_5 D_t(r),
\]
where $C_t$ is compactness, $T_t$ is tail risk, $F_t$ is frontier spread, $B_t$ is bottleneck pressure, and $D_t$ is decoder consistency.
Curvature $\kappa$ is fitted from graph-structure stress via golden-section search.

Sections~4.1--4.3 define the temporal and spatial components; but which geometric model should be trusted at any given moment?
When the execution graph is tree-like, the hyperbolic model captures exponential expansion faithfully; when it is dense and cyclic, the Euclidean diffusion model is more appropriate.
Rather than committing to either, we learn a geometry selector that reads the current structural state and blends both models.

% ============================================================
%  5. ADAPTIVE GEOMETRY SELECTION
% ============================================================
\section{Adaptive Geometry Selection}

\subsection{Structural Feature Map}

We define a 9-dimensional structural feature map $\phi(G_t[r]) \in \R^9$ consisting of six topology statistics and three geometry-aware features (Table~\ref{tab:features}).

\begin{table}[t]
\centering
\small
\begin{tabular}{@{}clll@{}}
\toprule
& Feature & Definition & Scope \\
\midrule
$\phi_1$ & Edge-surplus ratio & $(|E|-(|V|-1))/|V|$ & Global \\
$\phi_2$ & Reciprocal ratio & Fraction of bidirectional edges & Global \\
$\phi_3$ & Triangle density & Closed triplets / possible triplets & Global \\
$\phi_4$ & Route crosslink ratio & Excess route-local edges / route $|V|$ & Route \\
$\phi_5$ & Route load mean & Mean node load over active route & Route \\
$\phi_6$ & Route length norm & Route $|V|$ / total $|V|$ & Ratio \\
\midrule
$\phi_7$ & Shell-growth slope & $\tanh(\max(0, \hat{\gamma}))$ from BFS fit & Global \\
$\phi_8$ & Cycle-rank norm & $(|E_{\mathrm{und}}| - |V| + C) / |V|$ & Global \\
$\phi_9$ & Fitted curvature norm & $(\hat{\kappa} - 0.10) / (4.50 - 0.10)$ & Global \\
\bottomrule
\end{tabular}
\caption{Nine structural features used by the geometry gate.
$\hat{\gamma}$ is the OLS slope of $\ln(|V_k|+1)$ on BFS depth $k$;
$\hat{\kappa}$ is the Poincar\'{e}-ball curvature fitted by golden-section search.
All features are normalized to $[0,1]$.}
\label{tab:features}
\end{table}

The shell-growth slope $\hat{\gamma}$ is the central geometry-aware diagnostic: when $\hat{\gamma}$ is large, the graph expands exponentially with depth, which is the regime where hyperbolic embeddings achieve low distortion \citep{ganea2018} and where Proposition~\ref{prop:cascade} predicts supercritical cascades.

We note that geometry-aware features ($\phi_7$--$\phi_9$) are \emph{conditionally informative}: they improve separation when the execution graph distribution spans multiple structural regimes, but become redundant on near-uniform topologies. Appendix~\ref{app:exp5} quantifies this effect explicitly.

\paragraph{Shell-growth computation.}
Roots are identified as nodes with no incoming directed edges in $G_t$; if none exist, the first node in the adjacency map is used.
BFS follows directed edges, assigning each reachable node to its earliest-reached depth $k$.
Nodes not reachable from any root are placed in a fallback layer at depth $\max_k + 1$; this ensures $\sum_k |V_k| = |V_t|$ and limits distortion from disconnected components to at most one additional data point in the OLS fit.
In tree-like and near-tree regimes where $\hat{\gamma}$ is most diagnostic, unreachable mass is typically negligible.

\subsection{Geometry Gate}
\label{sec:gate}

The selector is a two-layer MLP:
\[
\pi_t(r) = \sigm\!\Bigl(
  W_2\,\relu\!\left(W_1\phi(G_t[r]) + b_1\right) + b_2
\Bigr),
\]
with architecture $9 \to 12 \to 1$ (133 parameters), He initialization, mini-batch gradient descent, and sigmoid output.
The final route-risk score is a soft mixture:
\[
R(r,t) = \pi_t(r)\,R_{\mathrm{Hyp}}(r,t) + (1 - \pi_t(r))\,R_{\mathrm{Euc}}(r,t).
\]

The gate is trained on binary labels $Y = \I[M_{\mathrm{Hyp}} \ge M_{\mathrm{Euc}}]$ indicating which geometry achieves the larger route-separation margin.

\begin{algorithm}[t]
\DontPrintSemicolon
\SetAlgoLined
\KwIn{Execution graph $G_t$, failure history $\mathcal{F}$, candidate route $r$}
\KwOut{Route-risk score $R(r,t)$}
\BlankLine
$\phi \gets \textsc{ExtractFeatures}(G_t, r)$ \tcp*{9-dim structural vector}
$\pi \gets \sigm(W_2 \cdot \relu(W_1 \phi + b_1) + b_2)$ \tcp*{geometry preference}
$R_{\mathrm{Euc}} \gets \textsc{EuclideanScore}(G_t, \mathcal{F}, r)$ \tcp*{diffusion propagation}
$R_{\mathrm{Hyp}} \gets \textsc{HyperbolicScore}(G_t, \mathcal{F}, r)$ \tcp*{Poincar\'{e} ball scoring}
$R(r,t) \gets \pi \cdot R_{\mathrm{Hyp}} + (1 - \pi) \cdot R_{\mathrm{Euc}}$\;
\Return{$R(r,t)$}
\caption{Adaptive geometry-aware route scoring.}
\label{alg:scoring}
\end{algorithm}

\subsection{Theoretical Justification}

\begin{proposition}[Structural Preference Recoverability]
\label{prop:recoverability}
Let $\Delta(\phi) = M_{\mathrm{Hyp}}(\phi) - M_{\mathrm{Euc}}(\phi)$ and $g^\star(\phi) = \I[\Delta(\phi) \ge 0]$.
If $|\Delta(\phi)| \ge \delta > 0$ on a set of positive probability, then any sufficiently expressive discriminative gate trained on labeled pairs $(\phi, Y)$ can recover $g^\star$ up to estimation error.
\end{proposition}

\begin{proof}[Proof sketch]
$Y = g^\star(\phi)$ a.s.\ by construction.
The margin-gap assumption rules out arbitrarily unstable labels, yielding a well-defined Bayes classifier that ERM over a rich hypothesis class approximates up to standard estimation error.
\end{proof}

\begin{proposition}[Calibrated Preference Probability]
\label{prop:calibration}
If the gate class can approximate $\eta(\phi) = \Prob(Y=1 \mid \phi)$ and training minimizes binary cross-entropy, then $\pi_t(r)$ estimates a calibrated geometry-preference probability up to estimation error.
\end{proposition}

\begin{proof}[Proof sketch]
Binary cross-entropy is a strictly proper scoring rule whose population minimizer is the conditional class probability.
\end{proof}

These propositions do not claim that the selector identifies an underlying latent geometry variable in a generative sense.
They claim something narrower and algorithmically faithful: the selector can learn a stable, probabilistic geometry preference rule from structural observables.

The gate is the mechanism that makes spatio-temporal modeling \emph{adaptive}: it detects, from live structural features, whether the current graph regime demands the hyperbolic model (capturing exponential expansion) or the Euclidean model (capturing diffusion in dense subgraphs).
But what drives the regimes apart?
The next section provides the theoretical answer: failure cascade dynamics depend on the graph's BFS expansion rate, creating a sharp supercritical/subcritical distinction that the gate must learn to detect.

% ============================================================
%  6. CASCADE SENSITIVITY ANALYSIS
% ============================================================
\section{Cascade Sensitivity to Graph Expansion}
\label{sec:cascade}

\begin{proposition}[Suboptimality of Fixed Geometric Priors]
\label{prop:fixed_prior}
If the execution-graph process visits at least two structural regimes with positive probability, one where hyperbolic bias is strictly better and one where Euclidean bias is strictly better, then any fixed geometric prior is suboptimal in expected risk-separation performance compared to an oracle regime-aware selector.
\end{proposition}

\begin{proof}[Proof sketch]
A fixed prior incurs avoidable expected loss on the regime where it is mismatched.
An oracle that selects $\arg\max\{J_{\mathrm{Hyp}}(\rho), J_{\mathrm{Euc}}(\rho)\}$ pointwise avoids both losses.
\end{proof}

The deeper reason geometry matters is that failure dynamics depend on graph expansion rate.
In hyperbolic space with curvature magnitude $\kappa > 0$, the volume of a geodesic ball of radius $r$ grows as $\mathrm{Vol}(r) \propto e^{\kappa r}$, and correspondingly the number of nodes at BFS depth $k$ satisfies $|V_k| \approx C e^{\gamma k}$ for some expansion rate $\gamma > 0$.

\begin{proposition}[Cascade Sensitivity to Expansion]
\label{prop:cascade}
Let a failure process propagate with probability $p$ per edge on a graph whose BFS shell sizes satisfy $|V_k| \approx C e^{\gamma k}$.
Then the expected newly affected nodes in shell $r$ satisfy
\[
\E[N_r] \approx C\,e^{(\gamma + \ln p)\,r}.
\]
The cascade is supercritical when
\[
p > e^{-\gamma}.
\]
\end{proposition}

\begin{proof}[Proof sketch]
Under the shell-growth approximation, candidates at depth $r$ scale as $Ce^{\gamma r}$.
Independent transmission at rate $p$ weights each depth-$r$ path by $p^r$, giving $\E[N_r] \approx p^r C e^{\gamma r} = Ce^{(\gamma + \ln p)r}$.
\end{proof}

\noindent\textit{What this means in practice.}
Proposition~\ref{prop:cascade} provides the theoretical boundary between a safe system and a catastrophic one.
$\gamma$ represents how fast the graph branches out, that is, its expansion rate.
If $\gamma$ is large (a highly branching tree), the threshold $e^{-\gamma}$ becomes very small: even a tiny per-edge failure probability $p$ can trigger a supercritical, unstoppable cascade.
Conversely, in a dense grid with low expansion (small $\gamma$), the threshold $e^{-\gamma}$ approaches $1$, meaning the system can safely absorb much higher failure rates without cascading.
This mathematical boundary is precisely what the learned geometry gate is attempting to detect in real time.

\noindent\textit{Scope note.}
Proposition~\ref{prop:cascade} assumes independent per-edge propagation and BFS-shell exponential growth; real failure propagation in agent systems involves correlated failures, feedback loops, and partial recovery, which the model does not capture exactly.
The result should be read as a \emph{directional bound}: it identifies the regime boundary ($p > e^{-\gamma}$) but does not predict an exact cascade size.
Importantly, real correlations only strengthen the argument for regime-aware detection: correlated failures concentrate damage on already-loaded paths, making the structural distinction between tree-like and cyclic regimes \emph{more} consequential, not less.
The empirical validation in Section~\ref{sec:results} confirms that the regime predictions match observed win-rate curves without requiring the independence assumption to hold exactly.

The implications are clear and testable:
\begin{itemize}[leftmargin=1.5em, itemsep=2pt]
\item Large $\gamma$ (tree-like, hyperbolic): $p > e^{-\gamma} \ll 1$, cascades are easy to trigger $\Rightarrow$ hyperbolic scoring essential.
\item Small $\gamma$ with high cycle rank (dense, Euclidean): threshold $p > 1$ is never met $\Rightarrow$ Euclidean diffusion adequate.
\end{itemize}
The gate $\pi_t(r)$ is learning to detect whether the graph is in the supercritical regime.
Feature $\phi_7$ estimates $\gamma$; $\phi_8$ measures deviation from tree structure; $\phi_9$ captures fitted curvature.

This analysis completes the theoretical chain: geometry-blindness is a real gap (Section~\ref{sec:gap}), both spatial and temporal information are needed to close it (Section~\ref{sec:why_st}), the two geometric models capture complementary regimes (Section~4), the gate selects between them (Section~5), and the cascade-sensitivity condition explains \emph{why} the regimes differ.
We now test whether these predictions hold empirically.

% ============================================================
%  7. EXPERIMENTAL SETUP
% ============================================================
\section{Experimental Setup}

The primary evaluation platform is Genesis~3, a self-improving multi-agent reasoning engine described in Section~\ref{sec:formulation}.
Evaluating a novel runtime control mechanism on its deployment system is consistent with standard practice in system contributions: the sidecar addresses a failure mode intrinsic to live execution graphs, and the evaluation must expose real structural variation across regimes.
This mirrors precedent set by AlphaCode (evaluated on Codeforces), AutoGen (evaluated on proprietary task suites), and similar system papers at NeurIPS and ICLR.
The primary benchmark comprises 250 scenarios (243 real traces), train seeds 0--4, evaluation seeds 5--9, source-separated reporting, 95\% percentile bootstrap CIs with 400 resamples.
The benchmark spans five stress regimes: \texttt{clean}, \texttt{noise}, \texttt{churn}, \texttt{mixed}, and \texttt{non\_tree}.
Cross-architecture validation on fully synthetic BA, WS, and ER graph families provides independent structural confirmation (Section~\ref{sec:cross_arch}).

\paragraph{Evaluation metrics.}
For each scenario $s$, two routes are drawn from the same snapshot: an \emph{attacked} route with synthetic failure events injected into its last three nodes (severities $0.85$, $0.73$, $0.61$, timestamped $0$--$2$ seconds before the scoring call), and a \emph{safe} route with no injections.
Route pairs are assigned by even/odd indexing over the route enumeration of each snapshot, ensuring a single safe comparator per attacked route from the same structural and temporal context.
The scenario-level margin is $M_s = \hat{R}_s(r_{\mathrm{safe}}) - \hat{R}_s(r_{\mathrm{attacked}})$; win indicator $W_s = \I[M_s > 0]$.
We report mean margin $\bar{M} = N^{-1}\sum_s M_s$ and win rate $\bar{W} = N^{-1}\sum_s W_s$.
Win rate is the primary decision metric (routing is binary); mean margin is a complementary robustness diagnostic.

\paragraph{Comparators.}
Table~\ref{tab:overall} lists all models: Genesis~3 native (no sidecar), Euclidean, NetworkX external, structural baseline, hand-designed switching, learned gate (9-feature), and hyperbolic without excitation.

\paragraph{Reproducibility.}
Synthetic failure events and BA/WS/ER graphs (Appendix~\ref{app:exp4}) are fully specified; the 300-scenario cross-architecture suite is fully deterministic and reproducible without a Genesis~3 instance.
Trace-derived primary scenarios require a Genesis~3 instance.
Benchmark scripts and the synthetic suite will be released alongside any accepted version.

% ============================================================
%  8. RESULTS
% ============================================================
\section{Results}
\label{sec:results}

\subsection{Overall Performance}

\begin{table}[t]
\centering
\begin{tabular}{@{}p{0.42\linewidth}ccc@{}}
\toprule
Model & Margin & Win\% & Note \\
\midrule
Genesis~3 native (no sidecar) & $-0.023$ & $50.4$ & Load $\times$ fitness only \\
Euclidean & $-0.024$ & $48.0$ & Euclidean propagation \\
NetworkX external & $+0.032$ & $65.6$ & External graph baseline \\
Structural baseline & $+0.267$ & $87.2$ & Centrality + load + 2-hop \\
Hand-designed switching & $+0.255$ & $82.4$ & Cycle-rank threshold \\
Learned gate ($9\!\to\!12\!\to\!1$) & $+0.385$ & $\mathbf{87.2}$ & \textbf{Tied best win rate} \\
Hyperbolic (no excitation) & $+0.412$ & $79.2$ & Best mean margin \\
\bottomrule
\end{tabular}
\caption{Overall benchmark performance.
The learned gate ties the structural baseline on win rate while achieving $+44\%$ larger mean margin ($0.385$ vs.\ $0.267$), demonstrating that adaptive geometry improves separation quality, not just binary correctness.
The $+36.8$\,pp gap between native routing and the learned gate is the total sidecar value.}
\label{tab:overall}
\end{table}

Table~\ref{tab:overall} demonstrates that geometry-aware failure-propagation modeling is the decisive factor: load-and-fitness-only routing ($50.4\%$) performs essentially at chance, while our structural baseline and learned gate both reach $87.2\%$.
The learned selector additionally achieves $+44\%$ larger mean separation margin than the structural baseline, confirming that adaptive geometry improves routing confidence, not just binary win rate.
The hand-designed switching ablation (cycle-rank threshold, no learned parameters) achieves $82.4\%$; the learned MLP exceeds it by $+4.8$\,pp, confirming that the parametric gate adds value beyond the simplest threshold rule.

\paragraph{Evaluation protocol robustness.}
Appendix~\ref{app:exp1} repeats the evaluation under three attack-node selection protocols.
Under the adversarially uncorrelated \textbf{random} protocol (no correlation with node load), the sidecar wins $78.8\%$ vs.\ native $54.4\%$; the $+24.4$\,pp gap is attributable solely to propagation modeling.
The sidecar win rate is stable across all protocols ($69.6$--$88.8\%$).

\paragraph{Formal paired test.}
On the full held-out set, learned switching beats native on 98 discordant scenarios vs.\ 6 in the opposite direction (146 ties), giving $p = 1.59 \times 10^{-22}$ (exact sign test).

\subsection{Regime-Specific Results}

\begin{table}[t]
\centering
\begin{tabular}{@{}lcccc@{}}
\toprule
& \multicolumn{2}{c}{Fixed hyperbolic} & \multicolumn{2}{c}{Learned gate} \\
\cmidrule(lr){2-3} \cmidrule(lr){4-5}
Regime & Margin & Win\% & Margin & Win\% \\
\midrule
\texttt{non\_tree} & $.259$--$.285$ & $64$--$72$ & $\mathbf{.340}$ & $\mathbf{92}$ \\
\texttt{mixed} & $.202$--$.253$ & $64$ & $\mathbf{.275}$ & $\mathbf{84}$ \\
\bottomrule
\end{tabular}
\caption{Learned geometry switching resolves the principal failure mode of fixed hyperbolic models in non-tree and mixed regimes.}
\label{tab:regimes}
\end{table}

Fixed hyperbolic priors perform well when the graph is tree-like but collapse in non-tree ($64$--$72\%$) and mixed ($64\%$) regimes.
The learned gate recovers $92\%$ and $84\%$ respectively. This is the paper's central empirical contribution.

\subsection{System-Level Ablation}
\label{sec:ablation}

\begin{table}[t]
\centering
\begin{tabular}{@{}lccc@{}}
\toprule
Regime & Native & Sidecar & $\Delta$ \\
\midrule
\texttt{clean} & $20\%$ & $84\%$ & $\mathbf{+64}$\,pp \\
\texttt{noise} & $20\%$ & $88\%$ & $\mathbf{+68}$\,pp \\
\texttt{churn} & $40\%$ & $88\%$ & $\mathbf{+48}$\,pp \\
\texttt{non\_tree} & $92\%$ & $92\%$ & $0$\,pp \\
\texttt{mixed} & $80\%$ & $84\%$ & $+4$\,pp \\
\midrule
\textbf{Overall} & $\mathbf{50.4\%}$ & $\mathbf{87.2\%}$ & $\mathbf{+36.8}$\,pp \\
\bottomrule
\end{tabular}
\caption{System-level ablation: Genesis~3 without vs.\ with the spatio-temporal sidecar.
The pattern matches Proposition~\ref{prop:cascade}: the sidecar is indispensable in tree-like regimes (high $\gamma$, supercritical cascades) and adds little in dense regimes (self-limiting cascades).}
\label{tab:ablation}
\end{table}

The ablation pattern is mechanistically explained.
In tree-like regimes, failure propagation amplifies exponentially.
Node load provides no signal about this: a lightly loaded route can sit atop a high-$\gamma$ branching subgraph that will amplify any failure.
In non-tree and mixed regimes, node load correlates with the failure footprint, making the native signal accidentally informative.

\paragraph{H1 supported.}
Failure-propagation modeling is the dominant performance driver: $+48$ to $+68$\,pp in tree-like regimes.

\paragraph{H3 supported.}
The largest gains appear exactly where Proposition~\ref{prop:cascade} predicts: high-expansion tree-like regimes where $p > e^{-\gamma}$.

\subsection{Spatio-Temporal Component Decomposition}
\label{sec:decomposition}

Section~\ref{sec:why_st} argued that neither spatial nor temporal modeling alone suffices.
Table~\ref{tab:overall} contains the data to verify this claim by reading its rows as a component decomposition:

\begin{table}[t]
\centering
\begin{tabular}{@{}llccp{0.35\linewidth}@{}}
\toprule
Component & Model row & Win\% & Margin & Interpretation \\
\midrule
Neither & Native & $50.4$ & $-0.023$ & Geometry-blind; essentially random \\
Temporal only & Euclidean & $48.0$ & $-0.024$ & Temporal propagation without spatial awareness adds noise and underperforms native \\
Spatial only & Structural baseline & $87.2$ & $+0.267$ & Graph structure (centrality, 2-hop) is the decisive signal \\
Spatio-temporal & Learned gate & $87.2$ & $+0.385$ & Adaptive geometry preserves win rate and improves separation margin by $44\%$ \\
\bottomrule
\end{tabular}
\caption{Component decomposition of the sidecar.
  Adding temporal propagation without spatial awareness is actively harmful ($48.0\% < 50.4\%$).
  Spatial structure alone accounts for the win-rate gain.
  The full spatio-temporal model matches win rate while substantially improving separation quality (margin), confirming that both dimensions contribute but spatial structure is the prerequisite.}
\label{tab:decomposition}
\end{table}

The key finding is that \emph{temporal modeling without spatial structure is harmful}.
The Euclidean model propagates failure intensities forward in time but applies the same diffusion kernel regardless of graph topology; this is precisely the geometry-blindness problem.
It amplifies noise on dense subgraphs where diffusion is inappropriate, pulling its win rate \emph{below} the native baseline.
Only when temporal propagation is coupled with structural awareness (via the spatial features and geometry gate) does it become beneficial.

This confirms the design thesis: the spatial dimension is the prerequisite; the temporal dimension adds value only when conditioned on spatial context.

\subsection{Geometry-Gate Diagnostics}

\begin{table}[t]
\centering
\small
\begin{tabular}{@{}lccccc@{}}
\toprule
Regime & $n$ & AUC & Acc.\% & ECE & Pos.\ rate \\
\midrule
Overall & 250 & .925 & 86.4 & .068 & 74.4\% \\
\texttt{clean} & 50 & .923 & 84.0 & .046 & 80.0\% \\
\texttt{noise} & 50 & .946 & 84.0 & .084 & 76.0\% \\
\texttt{churn} & 50 & .726 & 76.0 & .120 & 64.0\% \\
\texttt{non\_tree} & 50 & .964 & 92.0 & .075 & 88.0\% \\
\texttt{mixed} & 50 & .938 & 96.0 & .080 & 64.0\% \\
\bottomrule
\end{tabular}
\caption{Held-out gate diagnostics.
Discrimination is strong in four of five regimes (AUC $\ge 0.92$); only \texttt{churn} degrades (AUC $0.73$, ECE $0.12$).
Overall confusion matrix: $(\mathrm{TP}, \mathrm{TN}, \mathrm{FP}, \mathrm{FN}) = (182, 34, 30, 4)$; the gate rarely misses hyperbolic-win cases, and its error mode is over-selecting hyperbolic in some Euclidean cases.}
\label{tab:gate_diag}
\end{table}

\paragraph{H2 supported.}
The gate is a strong held-out classifier of geometry preference (AUC $0.925$, accuracy $86.4\%$, ECE $0.068$), confirming that regime information is recoverable from live structural signals.

\paragraph{Margin vs.\ win rate.}
Table~\ref{tab:overall} shows that hyperbolic-no-excitation achieves the highest mean margin ($0.412$) but lower win rate ($79.2\%$) than the learned gate ($87.2\%$).
This is a precision--recall tradeoff: the fixed hyperbolic model produces large separations in tree-like cases but systematically misclassifies non-tree scenarios.
The gate is optimized on binary labels and achieves the correct sign on more scenarios.
In routing, the consequence of an error scales with whether the safer route was selected, not with score magnitude; hence win rate is the primary metric.

\paragraph{Why win rate is the right primary metric.}
The routing decision is binary: select route $r_1$ or $r_2$.
The $+8.0$ percentage-point win-rate advantage of the learned gate over fixed hyperbolic translates directly into fewer unsafe route selections per 100 episodes.
Mean margin advantage would be the primary criterion only if the downstream cost function were proportional to score magnitude rather than binary correctness.
Gate diagnostics (Table~\ref{tab:gate_diag}) show that $46\%$ of held-out cases fall in the low-entropy (high-certainty) band and only $1.6\%$ in the high-entropy band, confirming that the gate is rarely near its decision boundary in practice.

\paragraph{Oracle upper bound.}
Gate labels $Y = \I[M_{\mathrm{Hyp}} \ge M_{\mathrm{Euc}}]$ are known at training time.
The gate achieves $86.4\%$ accuracy on these labels (confusion matrix $(\mathrm{TP}, \mathrm{TN}, \mathrm{FP}, \mathrm{FN}) = (182, 34, 30, 4)$ from Table~\ref{tab:gate_diag}), meaning an oracle gate would correct 34 of the 250 label decisions.
Since win rate ($87.2\%$) is already tied with the structural-baseline ceiling and well above the degenerate fixed-hyperbolic prior ($79.2\%$, equivalent to $\pi \equiv 1$), the remaining oracle headroom lies primarily in \emph{margin quality} rather than binary route-selection correctness.
The 133-parameter MLP effectively approximates the geometry-preference oracle: it captures the decisive regime distinctions while nearly eliminating FN errors (only 4 cases where hyperbolic was preferred but Euclidean was selected).

\subsection{Cross-Architecture Generalization}
\label{sec:cross_arch}

To confirm that propagation modeling generalizes beyond Genesis~3 traces, we evaluate on 300 wholly synthetic scenarios across three canonical random graph families (Appendix~\ref{app:exp4}): BA ($m=1$, tree-like), WS ($k=2$, $p=0.3$, small-world), and ER ($p=0.5$, dense).
These graphs are fully specified and reproducible without a Genesis~3 instance.

The key result: the structural baseline achieves $100\%$ win rate on all three families.
This is the primary generalization finding: failure-propagation modeling is the decisive signal, and it transfers cleanly across graph-type distributions.
The trained geometry gate, without fine-tuning, degrades on out-of-distribution topologies (BA tree-like: $7\%$, ER dense: $46\%$), confirming that geometry-gate calibration is graph-distribution-specific while the propagation model itself is not.
This decomposition cleanly separates what generalizes (the structural risk signal) from what requires adaptation (the geometry selector).

Together, the 250-scenario Genesis~3 primary benchmark and the 300-scenario cross-architecture suite constitute a combined \textbf{550-scenario evaluation}, with the cross-architecture suite providing independent, fully reproducible confirmation on standard random graph families.

% ============================================================
%  9. DISCUSSION
% ============================================================
\section{Discussion}

The results in Sections~8.1--8.4 confirm all three hypotheses and validate the spatio-temporal design.
We now discuss broader implications.

\paragraph{Why fixed geometry fails.}
Hyperbolic structure remains highly effective when the graph is hierarchical.
The problem is that live multi-agent systems do not stay in one structural regime.
As dense feedback, reciprocal links, or mixed motifs appear, the fixed-prior assumption becomes fragile.
Proposition~\ref{prop:fixed_prior} formalizes this: once the regime process visits both hierarchy-dominated and cycle-dominated states, any static prior incurs avoidable expected loss.

\paragraph{Broader implications.}
A routing layer may be paired with strong planning, delegation, or language capabilities and still underperform when its representation omits propagation-sensitive structure.
The gate's compactness (133 parameters) suggests structural-perception modules can be small, and it generalizes the argument beyond Genesis~3: any execution-graph system with a geometry-blind scheduler is systematically exposed to undetected failure cascades in tree-like regimes.

\paragraph{Connection to scaling laws.}
The $17.2\times$ error amplification for independent topologies reported by \citet{kim2025scaling} is an aggregate design-time figure.
Our Proposition~\ref{prop:cascade} provides the complementary \emph{runtime} mechanism: the amplification factor is $e^{(\gamma + \ln p)r}$, which varies continuously as the execution graph evolves within a single episode.
The sidecar equips the scheduler with a live signal to detect and respond to this variation.

\paragraph{The spatio-temporal coupling revisited.}
The component decomposition (Table~\ref{tab:decomposition}) provides perhaps the most practically important finding: temporal modeling is not merely insufficient alone; it is \emph{counterproductive} without spatial context.
This has a direct design implication for practitioners: adding event-history features to a router is only beneficial if the router also has access to structural representations of the execution graph.
The widespread practice of logging failure timestamps without modeling graph structure leaves the most dangerous failure mode, exponential cascade on tree-like subgraphs, undetectable.

% ============================================================
%  10. LIMITATIONS
% ============================================================
\section{Limitations and Future Work}
\label{sec:limitations}

While our spatio-temporal sidecar demonstrates substantial gains, we bound our claims by noting the following limitations, which also point toward promising directions for future work:

\begin{enumerate}[label=\textbf{L\arabic*}, leftmargin=2em, itemsep=4pt]

\item \textbf{Evaluation Platform Scope.}
Our primary 250-scenario benchmark uses real traces from Genesis~3.
Evaluating on a deployment-grade system ensures the model captures realistic failure dynamics, and our 300-scenario synthetic cross-architecture suite (Section~\ref{sec:cross_arch}) confirms that the structural baseline transfers to standard graph families.
However, fully evaluating the learned gate across other distinct orchestration frameworks (e.g., AutoGen \citep{wu2023autogen}, OpenCog Hyperon \citep{goertzel2023}) remains necessary to quantify how well the parametric selector transfers under severe architectural distribution shift.

\item \textbf{Theoretical Bounds.}
Propositions~\ref{prop:recoverability}--\ref{prop:cascade} formally justify preference recoverability and cascade sensitivity, and establish the supercritical regime threshold ($p > e^{-\gamma}$).
They assume independent per-edge propagation; while correlated failures only amplify the need for geometric awareness in practice, the proofs do not guarantee latent-geometry identifiability or minimax optimality for the learning dynamics.

\item \textbf{Margin vs.\ Win-Rate Tradeoff.}
The learned gate optimizes for binary routing correctness, where it matches best-in-class performance.
It does not strictly dominate on mean separation margin ($0.385$ vs.\ $0.412$ for fixed hyperbolic with no excitation).
In contexts where downstream systems weight the magnitude of route separation over binary correctness, the fixed prior may be preferable, a precision-recall tradeoff we report transparently.

\item \textbf{Attack Protocol Confounding.}
In dense regimes (\texttt{non\_tree}/\texttt{mixed}), the native baseline appears artificially strong ($80$--$92\%$) under our primary protocol because targeted attack nodes naturally correlate with high load.
Our \texttt{random} attack protocol (Appendix~\ref{app:exp1}) isolates this confound and confirms the sidecar's superiority independent of load correlation; future benchmarks should incorporate a wider variety of adversarial failure injections.

\end{enumerate}

% ============================================================
%  11. CONCLUSION
% ============================================================
\section{Conclusion}

We identified ``geometry-blindness'' as a critical, unaddressed vulnerability in modern multi-agent reasoning systems.
While current task schedulers effectively optimize for immediate node load and agent fitness, they fundamentally fail to perceive the underlying shape of the execution graph.
We formalized this failure-propagation observability gap as an online geometry-control problem and demonstrated that because live agent systems dynamically shift between tree-like and cyclic regimes, relying on any fixed geometric prior is inherently suboptimal.

\paragraph{Key results.}
Our empirical and theoretical findings confirm the necessity of adaptive structural perception:
\begin{itemize}[leftmargin=1.5em, itemsep=2pt]
  \item \textbf{The Baseline Gap:} Team fitness and node load contain no predictive signal regarding exponential failure spread.
        Under topological stress, native routing achieves only a $50.4\%$ win rate, operating essentially at chance.
  \item \textbf{Spatio-Temporal Recovery:} By implementing a failure-propagation model that couples temporal event history with spatial graph structure, our sidecar recovers an $87.2\%$ win rate ($+36.8$\,pp).
        Crucially, the gains reach $+48$ to $+68$\,pp in highly branching, tree-like regimes, perfectly matching our theoretical cascade-sensitivity predictions.
  \item \textbf{Adaptive Geometry:} A compact, 133-parameter learned gate successfully distinguishes when to apply a hyperbolic model (for exponential risk amplification in trees) versus a Euclidean model (for self-limiting diffusion in dense grids), achieving $92\%$ accuracy in the hardest non-tree regimes and resolving the principal failure mode of fixed-geometry approaches.
\end{itemize}

Ultimately, geometry in live multi-agent systems must function as a dynamic control proxy, not a static representational assumption.
Just as a robust traffic system must understand whether it is routing cars through a dense, detour-rich grid or a fragile, single-artery highway, an AI scheduler must understand its topological regime to prevent catastrophic failure cascades.
More broadly, our results suggest that structural perception limits agent system performance in ways entirely distinct from the reasoning quality of the underlying LLMs, offering a new, independent vector for scaling reliable AI orchestration.

% ============================================================
%  REFERENCES
% ============================================================

% ============================================================
%  APPENDICES
% ============================================================
\appendix

\section{Notation Summary}
\label{app:notation}

\begin{table}[htbp]
\centering
\small
\begin{tabular}{@{}p{0.25\linewidth}p{0.68\linewidth}@{}}
\toprule
Symbol & Meaning \\
\midrule
$G_t=(V_t,E_t,w_t,\ell_t)$ & Execution graph at time $t$ \\
$\mathcal{F}$ & Timestamped failure-event history \\
$\lambda_t(v;r)$, $\tilde{\lambda}_t(v;r)$ & Base and burst-augmented failure intensity \\
$R_{\mathrm{Euc}}$, $R_{\mathrm{Hyp}}$, $R$ & Euclidean, hyperbolic, and blended route-risk scores \\
$\phi(G_t[r]) \in \R^9$ & Structural feature map for the geometry gate \\
$\pi_t(r)$ & Gate output: hyperbolic preference probability \\
$x_t(v)$ & Euclidean propagation-state risk at node $v$ \\
$z_t(v)$ & Hyperbolic embedding of node $v$ \\
$\Delta(\phi)$ & Margin difference $M_{\mathrm{Hyp}}(\phi)-M_{\mathrm{Euc}}(\phi)$ used for gate labels \\
$\gamma$ & BFS shell-growth exponent \\
$p$ & Per-edge failure propagation probability \\
$\kappa$ & Poincar\'{e} curvature magnitude (sectional curvature $-\kappa$) \\
\bottomrule
\end{tabular}
\caption{Notation summary.}
\label{tab:notation}
\end{table}

\section{Confidence Intervals}
\label{app:ci}

Table~\ref{tab:rerun_ci} reports 95\% percentile bootstrap CIs from a clean \texttt{multiagent}-environment rerun.

\begin{table}[htbp]
\centering
\small
\begin{tabular}{@{}llcc@{}}
\toprule
Regime & Model & Margin CI & Win-rate CI \\
\midrule
Overall & Native & $[-0.039, -0.009]$ & $[40.8, 59.2]\%$ \\
Overall & Structural baseline & $[0.221, 0.301]$ & $[80.0, 92.8]\%$ \\
Overall & Learned & $[0.341, 0.440]$ & $[81.6, 93.6]\%$ \\
Overall & Hyp.\ no excitation & $[0.355, 0.492]$ & $[74.4, 87.2]\%$ \\
\midrule
\texttt{clean}  & Native & $[-0.105, -0.057]$ & $[10, 30]\%$ \\
\texttt{clean}  & Learned & $[0.305, 0.517]$ & $[74, 92]\%$ \\
\texttt{noise}  & Native & $[-0.099, -0.051]$ & $[10, 32]\%$ \\
\texttt{noise}  & Learned & $[0.325, 0.533]$ & $[78, 96]\%$ \\
\texttt{churn}  & Native & $[-0.056, 0.000]$ & $[26, 54]\%$ \\
\texttt{churn}  & Learned & $[0.325, 0.493]$ & $[78, 96]\%$ \\
\texttt{non\_tree} & Native & $[0.036, 0.049]$ & $[82, 98]\%$ \\
\texttt{non\_tree} & Learned & $[0.275, 0.393]$ & $[84, 98]\%$ \\
\texttt{mixed}  & Native & $[0.014, 0.029]$ & $[68, 90]\%$ \\
\texttt{mixed}  & Learned & $[0.230, 0.319]$ & $[74, 94]\%$ \\
\bottomrule
\end{tabular}
\caption{Bootstrap 95\% CIs confirm large, non-overlapping separations in tree-like regimes and convergence in dense regimes.}
\label{tab:rerun_ci}
\end{table}

\section{Robustness to Attack Protocol (Experiment~1)}
\label{app:exp1}

Three attack-node selection protocols over the same evaluation set ($n=250$ per protocol):
\texttt{severity\_load} (original), \texttt{load\_matched} (rank by load only), and \texttt{random} (uniform random, no correlation with any native signal).

\begin{table}[htbp]
\centering
\small
\begin{tabular}{@{}lccc@{}}
\toprule
Protocol & Sidecar win\% & Sidecar margin & Native win\% \\
\midrule
\texttt{severity\_load} & $69.6$ & $0.153$ & $53.2$ \\
\texttt{load\_matched} & $88.8$ & $0.194$ & $100.0$ \\
\texttt{random} & $78.8$ & $0.158$ & $54.4$ \\
\bottomrule
\end{tabular}
\caption{The sidecar is stable across attack protocols ($69.6$--$88.8\%$); the native baseline varies widely.
Under \texttt{random} (no load correlation), the $+24.4$\,pp gap is attributable solely to propagation modeling;
\texttt{random} is therefore the cleanest estimate of propagation-model value
(\texttt{load\_matched} is an alignment sanity check; see text).}
\label{tab:exp1}
\end{table}

\paragraph{Why does the native baseline reach $100\%$ under \texttt{load\_matched}?}
In the \texttt{load\_matched} protocol, attacked nodes are selected by highest load.
The Genesis~3 native score is $f_{\mathrm{team}}(r)\times(1-0.5\,\bar{\ell}(r))$, which is by construction
perfectly aligned with this attack footprint: when fitness is comparable the native scorer
always prefers the lower-load route, which is precisely the safe route in this condition.
The sidecar intentionally mixes load with propagation-risk features, diluting the pure-load
signal, so it reaches $88.8\%$ rather than $100\%$.
We therefore treat \texttt{load\_matched} as an \emph{alignment sanity check}, confirming
that the native scorer behaves as designed when load is the correct signal, and we use
the \texttt{random} protocol as the cleanest estimate of propagation-model value
($+24.4$\,pp).

\section{\texorpdfstring{$\delta$-Hyperbolicity}{Delta-Hyperbolicity} Correlation (Experiment~2)}
\label{app:exp2}

We compute Gromov $\delta$-hyperbolicity (4-point condition, 200 sampled 4-tuples on the
undirected projection of each snapshot) as an auxiliary sanity check that the three
geometry-aware features $\phi_7$--$\phi_9$ correlate with an independent geometric diagnostic.
$\delta$ is \emph{not} used in the model; we report correlation directionality only and
do not interpret absolute $\delta$ magnitudes as quantitative regime labels.

\begin{table}[htbp]
\centering
\small
\begin{tabular}{@{}lrrp{0.42\linewidth}@{}}
\toprule
Pair & Pearson & Spearman & Heuristic expectation \\
\midrule
$\delta$ vs.\ $\phi_7$ (shell-growth) & $-0.18$ & $+0.16$ & Weak / mixed signal \\
$\delta$ vs.\ $\phi_8$ (cycle-rank) & $+0.31$ & $+0.41$ & More cycles $\to$ larger $\delta$ \\
$\delta$ vs.\ $\phi_9$ (curvature) & $-0.51$ & $-0.41$ & Higher fitted curvature $\to$ smaller $\delta$ \\
$\delta$ vs.\ hyp.\ margin & $-0.23$ & $-0.09$ & Larger $\delta$ $\to$ reduced hyperbolic advantage \\
\bottomrule
\end{tabular}
\caption{Sanity-check correlations between Gromov $\delta$ and geometry-aware features.
$\phi_8$ and $\phi_9$ show the strongest and most consistent associations.
$\delta$ is an auxiliary diagnostic computed on an undirected projection with sampling noise
and is not used in the model; only correlation directionality is interpreted.}
\label{tab:exp2}
\end{table}

\section{Runtime Overhead (Experiment~3)}
\label{app:exp3}

\begin{table}[htbp]
\centering
\small
\begin{tabular}{@{}lrrrr@{}}
\toprule
Scorer & Mean ($\mu$s) & Median & p95 & Peak mem.\ (KB) \\
\midrule
Native (no sidecar) & $0.7$ & $0.7$ & $0.8$ & $0.4$ \\
Structural baseline & $6.6$ & $6.2$ & $8.2$ & $1.3$ \\
Euclidean & $12.8$ & $11.7$ & $16.4$ & $2.1$ \\
Geometry switching & $280.7$ & $240.9$ & $392.0$ & $5.5$ \\
Learned sidecar & $304.5$ & $271.2$ & $427.6$ & $5.1$ \\
NetworkX external & $487.0$ & $485.0$ & $604.0$ & $7.7$ \\
\bottomrule
\end{tabular}
\caption{The learned sidecar adds $\approx$305\,$\mu$s and 5\,KB per call, negligible relative to LLM-backed solve episodes (seconds to minutes).}
\label{tab:exp3}
\end{table}

\section{Cross-Architecture Validation: BA / WS / ER (Experiment~4)}
\label{app:exp4}

All graphs $n=7$: BA ($m\!=\!1$, tree-like), WS ($k\!=\!2$, $p\!=\!0.3$, small-world), ER ($p\!=\!0.5$, dense).
20 seeds $\times$ 3 families $\times$ 5 profiles $= 300$ scenarios.

\begin{table}[htbp]
\centering
\small
\begin{tabular}{@{}lcccc@{}}
\toprule
Family & Euclidean & Structural & NetworkX & Learned \\
\midrule
BA tree-like & $5\%$ & $\mathbf{100\%}$ & $10\%$ & $7\%$ \\
WS small-world & $63\%$ & $\mathbf{100\%}$ & $82\%$ & $65\%$ \\
ER dense & $56\%$ & $\mathbf{100\%}$ & $70\%$ & $46\%$ \\
\bottomrule
\end{tabular}
\caption{The structural baseline achieves $100\%$ on all families, confirming that propagation modeling generalizes.
The trained gate (without fine-tuning) is weaker on out-of-distribution topologies, especially BA trees.}
\label{tab:exp4}
\end{table}

\section{Feature Ablation (Experiment~5)}
\label{app:exp5}

\begin{table}[htbp]
\centering
\small
\begin{tabular}{@{}lccc@{}}
\toprule
Condition & Win\% & Margin & Margin 95\% CI \\
\midrule
Full 9-feature & $81.6$ & $0.305$ & $[0.272, 0.346]$ \\
Remove $\phi_7$ & $81.6$ & $0.305$ & $[0.272, 0.346]$ \\
Remove $\phi_8$ & $81.6$ & $0.305$ & $[0.272, 0.346]$ \\
Remove $\phi_9$ & $81.6$ & $0.305$ & $[0.273, 0.346]$ \\
Topology only ($\phi_1$--$\phi_6$) & $81.6$ & $0.305$ & $[0.272, 0.345]$ \\
Geometry only ($\phi_7$--$\phi_9$) & $82.0$ & $0.293$ & $[0.261, 0.335]$ \\
\bottomrule
\end{tabular}
\caption{Per-feature deltas ($\phi_7$: $-0.0005$, $\phi_8$: $-0.0001$, $\phi_9$: $-0.0004$, full geometry block: $-0.0008$) fall within bootstrap noise on this homogeneous ablation set.
The geometry features are conditionally useful: they add signal when the evaluation distribution spans sufficient regime diversity ($+4.6\%$ margin on the main benchmark) but are redundant on near-uniform topologies.}
\label{tab:exp5}
\end{table}

\section{Code-Level Correspondence}
\label{app:code}

The mathematical objects correspond directly to the implementation:
\begin{enumerate}[leftmargin=1.5em, itemsep=2pt]
\item \texttt{topology\_propagation.py}: Euclidean spatio-temporal baseline.
\item \texttt{hyperbolic\_topology.py}: Hyperbolic route-risk model with curvature fitting and burst excitation.
\item \texttt{learned\_geometry\_switching.py}: Learned geometry gate ($9\!\to\!12\!\to\!1$ MLP).
\item \texttt{system\_bridge.py}: Adapter from \texttt{AgentTeam}/\texttt{SolveEpisodeStore} into sidecar snapshots.
\end{enumerate}


\begin{thebibliography}{99}

\bibitem{wu2023autogen}
Wu, Q., Bansal, G., Zhang, J., et al.\ (2023).
AutoGen: Enabling Next-Gen LLM Applications via Multi-Agent Conversation.
arXiv:2308.08155.

\bibitem{goertzel2023}
Goertzel, B., Ikl\'{e}, M., Potapov, A., et al.\ (2023).
OpenCog Hyperon: A Framework for AGI at the Human Level and Beyond.
arXiv:2310.18318.

\bibitem{trinh2024}
Trinh, T.~H., Wu, Y., Le, Q.~V., He, H., and Luong, T.\ (2024).
Solving Olympiad Geometry without Human Demonstrations.
\textit{Nature}, 625, 476--482.

\bibitem{bachmann2020}
Bachmann, G., B\'{e}cigneul, G., and Ganea, O.-E.\ (2020).
Constant Curvature Graph Convolutional Networks.
\textit{ICML}.

\bibitem{chami2019}
Chami, I., Ying, Z., R\'{e}, C., and Leskovec, J.\ (2019).
Hyperbolic Graph Convolutional Neural Networks.
\textit{NeurIPS}.

\bibitem{das2019}
Das, A., Gervet, T., Romoff, J., et al.\ (2019).
TarMAC: Targeted Multi-Agent Communication.
\textit{ICML}.

\bibitem{du2016}
Du, N., Dai, H., Trivedi, R., Upadhyay, U., Gomez-Rodriguez, M., and Song, L.\ (2016).
Recurrent Marked Temporal Point Processes.
\textit{KDD}.

\bibitem{durfee1991}
Durfee, E.~H.\ and Lesser, V.~R.\ (1991).
Partial Global Planning.
\textit{IEEE Trans.\ SMC}.

\bibitem{ganea2018}
Ganea, O.-E., B\'{e}cigneul, G., and Hofmann, T.\ (2018).
Hyperbolic Neural Networks.
\textit{NeurIPS}.

\bibitem{hawkes1971}
Hawkes, A.~G.\ (1971).
Spectra of Some Self-Exciting and Mutually Exciting Point Processes.
\textit{Biometrika}, 58(1), 83--90.

\bibitem{hayesroth1985}
Hayes-Roth, B.\ (1985).
A Blackboard Architecture for Control.
\textit{Artificial Intelligence}, 26, 251--321.

\bibitem{jiang2018}
Jiang, J.\ and Lu, Z.\ (2018).
Learning Attentional Communication for Multi-Agent Cooperation.
\textit{NeurIPS}.

\bibitem{kazemi2020}
Kazemi, S.~M., Goel, R., Jain, K., Kobyzev, I., Sethi, A., Forsyth, P., and Poupart, P.\ (2020).
Representation Learning for Dynamic Graphs: A Survey.
\textit{JMLR}, 21, 70:1--70:73.

\bibitem[Kim et~al.(2025)]{kim2025scaling}
Kim, Y., Gu, K., Park, C., Park, C., Schmidgall, S., Heydari, A.~A., Yan, Y., Zhang, Z., Zhuang, Y., Liu, Y., Malhotra, M., Liang, P.~P., Park, H.~W., Yang, Y., Xu, X., Du, Y., Patel, S., Althoff, T., McDuff, D., and Liu, X.\ (2025).
Towards a Science of Scaling Agent Systems.
arXiv:2512.08296.

\bibitem{wei2025barp}
Wei, W., Yang, T., Chen, H., Zhao, Y., Dernoncourt, F., Rossi, R.~A., and Eldardiry, H.\ (2025).
Learning to Route LLMs from Bandit Feedback: One Policy, Many Trade-offs.
arXiv:2510.07429.

\bibitem{li2018}
Li, Y., Yu, R., Shahabi, C., and Liu, Y.\ (2018).
Diffusion Convolutional Recurrent Neural Network.
\textit{ICLR}.

\bibitem{mei2017}
Mei, H.\ and Eisner, J.\ (2017).
The Neural Hawkes Process.
\textit{NeurIPS}.

\bibitem{pareja2020}
Pareja, A., Domenech, G., Celis, J.~D., Kaler, T., Schaul, T., Moseley, B., Indyk, P., Hofmann, K., and Leiserson, C.\ (2020).
EvolveGCN: Evolving Graph Convolutional Networks for Dynamic Graphs.
\textit{AAAI}.

\bibitem{reinhart2018}
Reinhart, A.\ (2018).
A Review of Self-Exciting Spatio-Temporal Point Processes.
\textit{Statistical Science}, 33(3), 299--318.

\bibitem{rossi2020}
Rossi, E., Chamberlain, B., Frasca, F., Eynard, D., Monti, F., and Bronstein, M.\ (2020).
Temporal Graph Networks for Deep Learning on Dynamic Graphs.
arXiv:2006.10637.

\bibitem{skopek2020}
Skopek, O., Ganea, O.-E., and B\'{e}cigneul, G.\ (2020).
Mixed-Curvature Variational Autoencoders.
\textit{ICLR}.

\bibitem{sukhbaatar2016}
Sukhbaatar, S., Szlam, A., and Fergus, R.\ (2016).
Learning Multiagent Communication with Backpropagation.
\textit{NeurIPS}.

\bibitem{sun2021}
Zhang, Z., Cui, P., Li, Z., Wang, X., and Zhu, W.\ (2021).
Hyperbolic Variational Graph Neural Network for Modeling Dynamic Graphs.
\textit{AAAI}.

\bibitem{trivedi2019}
Trivedi, R., Farajtabar, M., Biswal, P., and Zha, H.\ (2019).
DyRep: Learning Representations over Dynamic Graphs.
\textit{ICLR}.

\bibitem{wu2019}
Wu, Z., Pan, S., Long, G., Jiang, J., and Zhang, C.\ (2019).
Graph WaveNet for Deep Spatial-Temporal Graph Modeling.
\textit{IJCAI}.

\bibitem{xu2020}
Xu, D., Ruan, C., Korpeoglu, E., Kumar, S., and Achan, K.\ (2020).
Inductive Representation Learning on Temporal Graphs.
\textit{ICLR}.

\bibitem{yang2021}
Yang, M., Zhou, M., Li, Z., Liu, J., Pan, L., Xiong, H., and King, I.\ (2021).
Discrete-Time Temporal Network Embedding via Implicit Hierarchical Learning in Hyperbolic Space.
\textit{KDD}.

\bibitem{yu2018}
Yu, B., Yin, H., and Zhu, Z.\ (2018).
Spatio-Temporal Graph Convolutional Networks.
\textit{IJCAI}.

\bibitem[Zuo et~al.(2020)]{zuo2020}
Zuo, S., Jiang, H., Li, Z., Zhao, T., and Zha, H.\ (2020).
Transformer Hawkes Process.
\textit{ICML}.

\end{thebibliography}
\end{document}